\documentclass[conference]{IEEEtran}
\IEEEoverridecommandlockouts
\usepackage{algpseudocode}
\usepackage{algorithm}

\usepackage{cite}
\usepackage{amsmath,amssymb,amsfonts}
\usepackage{amsthm}
\usepackage{graphicx}
\usepackage{textcomp}
\usepackage{caption, subcaption}
\usepackage{ulem}
\usepackage{booktabs}
\usepackage{listings}
\usepackage{colortbl}
\definecolor{royalpurple}{rgb}{0.47,0.32,0.66}
\usepackage{chngpage}
\usepackage{adjustbox}
\usepackage{setspace}
\usepackage{epstopdf}
\usepackage{multirow}
\usepackage{mathtools}
\usepackage{textcomp,booktabs}
\usepackage[table,xcdraw]{xcolor}
\usepackage{array}
\usepackage{booktabs}
\usepackage{makecell}

\usepackage{lipsum}
\newtheorem{definition}{Definition}
\newtheorem{theorem}{Theorem}

\title{Topology Generation of UAV Covert Communication Networks: A Graph Diffusion Approach with Incentive Mechanism}

\author{
\IEEEauthorblockN{
Xin Tang\textsuperscript{1}, 
Qian Chen\textsuperscript{2},
Fengshun Li\textsuperscript{1},
Youchun Gong\textsuperscript{1},
Yinqiu Liu\textsuperscript{3},
Wen Tian\textsuperscript{4},
Shaowen Qin\textsuperscript{1},
Xiaohuan Li\textsuperscript{1}\textsuperscript{*}
}

\IEEEauthorblockA{\textsuperscript{1}Guangxi University Key Laboratory of Intelligent Networking and Scenario System (\\
School of Information and Communication), Guilin University of Electronic Technology, Guilin, China}

\IEEEauthorblockA{\textsuperscript{2}School of Architecture and Transportation Engineering, Guilin University of Electronic Technology, Guilin, China}

\IEEEauthorblockA{\textsuperscript{3}College of Computing and Data Science, Nanyang Technological University, Singapore}

\IEEEauthorblockA{\textsuperscript{4}School of Electronic and Information Engineering, Nanjing University of Information Science and Technology, Nanjing, China}

\IEEEauthorblockA{
Email: \{tangx, chenqian\}@mails.guet.edu.cn, \{lfsguet, gycguet\}@163.com, yinqiu001@ntu.edu.sg,\\
csusttianwen@163.com, qinsw@mails.guet.edu.cn, *lxhguet@guet.edu.cn}
}

\begin{document}

\maketitle
\begin{abstract}
With the growing demand for Uncrewed Aerial Vehicle (UAV) networks in sensitive applications, such as urban monitoring, emergency response, and secure sensing, ensuring reliable connectivity and covert communication has become increasingly vital. However, dynamic mobility and exposure risks pose significant challenges. To tackle these challenges, this paper proposes a self-organizing UAV network framework combining Graph Diffusion-based Policy Optimization (GDPO) with a Stackelberg Game (SG)-based incentive mechanism. The GDPO method uses generative AI to dynamically generate sparse but well-connected topologies, enabling flexible adaptation to changing node distributions and Ground User (GU) demands. Meanwhile, the Stackelberg Game (SG)-based incentive mechanism guides self-interested UAVs to choose relay behaviors and neighbor links that support cooperation and enhance covert communication. Extensive experiments are conducted to validate the effectiveness of the proposed framework in terms of model convergence, topology generation quality, and enhancement of covert communication performance.

\end{abstract}


\section{Introduction}
Uncrewed Aerial Vehicles (UAVs) are expected to be widely deployed for urban inspection, emergency rescue, and regional sensing applications \cite{tang2023digital}. Consequently, establishing an efficient self-organizing cooperative UAV network in a dynamic and distributed environment has become a central challenge in both industry and academia.

Compared to traditional fixed infrastructure, UAV networks must handle highly mobile nodes, easily disrupted communication links, and sudden tasks, while maintaining robust global connectivity under varying scales and operational conditions \cite{tian2023uav}. The authors in \cite{gaydamaka2023dynamic} proposed distributed algorithms for UAV network topology reorganization and maintenance without external positioning, supporting merging and disjoining in GNSS-denied environments. The authors in \cite{wang2023throughput} proposed a unified UAV covert communication framework that jointly optimizes 3D flight paths and transmission power with robust anti-detection, causality, and collision avoidance constraints. The authors in \cite{wang2408generative} proposed a secure ISAC framework employing two diffusion models to activate links and generate pilot-masked signals, mitigating unauthorized channel state information sensing. Therefore, designing a self-adaptive, scalable, and partition-resilient topology generation method is crucial to improving network coordination capabilities.

The connection patterns, communication links, and relay behaviors within UAV networks can easily be exploited by eavesdroppers to perform traffic analysis, location tracking, or link inference during sensitive tasks such as information collection and area monitoring \cite{chen2023covert}. To this end, the authors in \cite{tang2025dnn} proposed a two-stage UAV swarm system with pipeline deep neural network task assignment and flight path planning using a multi-agent generative diffusion model-assisted deep deterministic policy gradient approach. Unfortunately, excessive direct neighbors, overly centralized connections, or predictable relay behaviors make the network size and individual roles more easily inferred by external observers, potentially leading to information leakage or task disruption. The authors in \cite{tian2025contract} proposed a UAV-assisted contract-theoretic model to incentivize honest covert data transmission tasks and improve throughput under information asymmetry constraints. However, existing approaches often neglect UAV network topology optimization and cannot ensure persistent connection. Therefore, beyond ensuring high connectivity, reducing the likelihood of node exposure to enhance the covert performance of multi-UAV cooperative communication has become an increasingly important requirement for UAV networks.

To address the dual challenges of connectivity and communication covertness in dynamic UAV networks, this paper proposes a self-organizing network framework that integrates Graph Diffusion-based Policy Optimization (GDPO) with a Stackelberg Game (SG)-based incentive mechanism. Leveraging the data generation capabilities of generative Artificial Intelligence \cite{liu2024cross}, the GDPO method dynamically constructs network topologies with high connectivity and sparsity, allowing the network to flexibly adapt to dynamic node distributions and varying demands \cite{liu2024graph}. Furthermore, a SG-based incentive mechanism is designed to guide UAV nodes to select appropriate relay behaviors and neighbor connections under self-interested conditions, thereby encouraging continuous cooperation and indirectly improving the covert characteristics of network communication.

The remainder of the paper is organized as follows. In Section \ref{sec:2}, we present the system model. A SG-based incentive mechanism for covert communication is proposed in Section \ref{sec:3}. Section \ref{sec:5} introduces the proposed approach based on GDPO for topology optimization. Section \ref{sec:6} illustrates the experimental results and analysis. Finally, Section \ref{sec:7} concludes this paper.


\section{System Model}
\label{sec:2}
Fig.~\ref{Fig01} illustrates a decentralized multi-UAV network designed to support high persistent connection and emergency networking of sensitive or covert data for Ground Users (GUs). The UAVs form a multi-hop communication topology to relay classified information from GUs to an aerial base station (such as an airship) while maintaining covert communication. The practical application scenarios include hydraulic and structural infrastructure detection (such as bridges and dams), transportation infrastructure detection, and emergency response in events such as earthquakes or traffic accidents near critical or restricted facilities. 

\subsection{Communication Model}
The communication link between the UAV and the GUs can be either Line-of-Sight (LoS) or Non-Line-of-Sight (NLoS). Without loss of generality, we adopt the Air-to-Ground (A2G) channel model proposed in~\cite{al2014optimal}. The probability of establishing a link between the UAV and GU \( i \) is given by:
\begin{equation}
\begin{cases}
p_{\text{LoS}} = \dfrac{1}{1 + \varphi \exp\left( -\lambda \left( \dfrac{180}{\pi} \tan^{-1} \left( \dfrac{h_j}{s_{(i,j)}} \right) - \varphi \right) \right)} \\
p_{\text{NLoS}} = 1 - p_{\text{LoS}}
\end{cases},
\label{eq:plos_joint}
\end{equation}
where \( \phi \) and \( \lambda \) are environment-related constants. \( h_j \) is the flying altitude of UAV \( j \). \( s_{(i,j)} \) represents the horizontal distance between the GU \( i \) and the UAV \( j \), calculated as \( s_{(i,j)} = \sqrt{(x_i - x_j)^2 + (y_i - y_j)^2} \). The 2D coordinate of the GU \( i \) is given by \( (x_i, y_i) \), where \( i \in \mathcal{I} \), \( \mathcal{I} = \{1, 2, \ldots, I\}\). \( (x_j, y_j) \) denotes the 2D coordinate of the UAV,  where \( j \in \mathcal{J} \), \( \mathcal{J} = \{1, 2, \ldots, J\}\).

In A2G environments, radio signals experience not only free-space path loss, but also additional attenuation caused by shadowing and multipath scattering effects. Since UAV deployment is typically based on long-term variations of the channel rather than instantaneous fluctuations~\cite{ghazzai2015optimized}, we focus on modeling the average path loss of the signal.
The path loss models for LoS and NLoS links (in dB) are given by $L_{\text{LoS}} = 20 \log \left( \frac{4\pi f_c d_{(i,j)}}{c} \right) + \varepsilon_{\text{LoS}}$ and $L_{\text{NLoS}} = 20 \log \left( \frac{4\pi f_c d_{(i,j)}}{c} \right) + \varepsilon_{\text{NLoS}}$,  respectively, where \( f_c \) is the carrier frequency, and \( d_{(i,j)} = \sqrt{h^2 + s_{(i,j)}^2} \) denotes the distance between the GU \(i\) and the UAV \( j \), \( \varepsilon \) represent the average additional loss. For analytical simplicity, the probabilistic average path loss is considered as $ L(h, s_{(i,j)}) = L_{\text{LoS}} \cdot p_{\text{LoS}} + L_{\text{NLoS}} \cdot p_{\text{NLoS}}$.
Given a transmission power \( P_{{tx_j}} \) of UAV $j$, the received power at the GU \( i \) is computed as
\begin{equation}
P_{r_i} = P_{{tx}_j} - L(h, s_{(i,j)}).
\label{eq:rxpower}
\end{equation}

To guarantee communication quality, \( P_{r_i} \) is needed to exceed a threshold \( P_{min} \). This requirement is equivalent to a path loss condition \( L(h, s_{(i,j)}) \leq L_{ {th}} \). The signal coverage radius is defined as \( \mathcal{R} = \mathbf{r} \big|_{L(h, s) = L_{ {th}}} \)~\cite{alzenad20173}.

\begin{figure}[tpb]
\centering
\includegraphics[width=8.8cm, height=7.2cm]{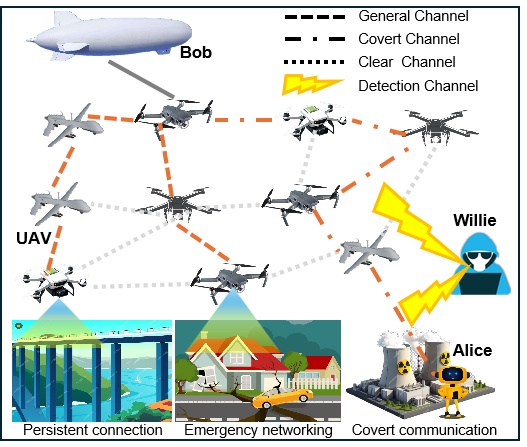}
\caption{System model.}
\label{Fig01}
\vspace{-0.6cm}
\end{figure}

\subsection{UAV Model}
To encourage participation in covert communication, each UAV receives a reward from GU, the message sender (such as Alice). The utility of UAV \( j \) is defined as
\begin{equation}
U_j = R_j - F_j,
\label{eq:utility}
\end{equation}
where $R_j=r_j\ln{\left(1+P_{{tx}_j}\right)}$ represents the reward received by the UAV from Alice, with $r_j$ denoting the unit reward, and $F_j=\phi_j{P_{{tx}_j}}$ denoting the transmission power consumption, and \( \phi_j \) is a constant. 

\subsection{Willie Model}
Willie decides whether Alice or the UAV is transmitting based on the received signal 
$\mathcal{Y}$, which is given by:
\begin{equation}
\mathcal{Y} =
\begin{cases}
n, & H_0, \\
P_{{tx}_j} g_j S_A + n, & H_1,
\end{cases}
\label{eq:hypothesis}
\end{equation}
where $n$ refers to additive white Gaussian noise with $(0, \sigma_{ {noise}}^2) $. $H_0$ denotes the null hypothesis, indicating that Alice does not transmit any signal. $H_1$ represents the alternative hypothesis, with an interaction occurring between Alice and the UAV. $S_A$ denotes the signal symbol transmitted by the UAV. $g_j$ represents the channel gain between the UAV and Willie. $\sigma_{noise}$ denotes the power of background noise.

Willie uses an energy detector to calculate \( |\mathcal{Y}|^2 \) and sets a threshold \( \epsilon \) for decision-making. Specifically, if \( |\mathcal{Y}|^2 > \epsilon \), it decides \( H_1 \). Otherwise, it decides \( H_0 \). Under each hypothesis, the statistical properties of \( |\mathcal{Y}|^2 \) are as follows: \( H_0 \) means that the signal contains only noise, \( \mathcal{Y} \sim \mathcal{N}(0, \sigma_{ {noise}}^2) \), that is \( |\mathcal{Y}|^2 \) follows an exponential or chi-squared distribution. In large-sample or high-SNR scenarios, it can be approximated by a Gaussian distribution via the central limit theorem. \( H_1 \) means that the signal energy is \( P_{{tx}_j} g_j^2 \), hence \( \mathcal{Y} \sim \mathcal{N}(P_{{tx}_j} g_j^2, \sigma_{ {noise}}^2) \), and likewise \( |\mathcal{Y}|^2 \) approximately follows a Gaussian distribution. The detection probability is $p_j = \mathbb{P}( {decide } H_1 \mid H_1) = \mathbb{P}(|\mathcal{Y}|^2 > \epsilon \mid H_1)$. Under the Gaussian approximation, the mean and variance of the detection statistic \( |\mathcal{Y}|^2 \) are $\mathbb{E}[|\mathcal{Y}|^2] = P_{{tx}_j} g_j^2 + \sigma_{ {noise}}^2 $ and $\mathbb{D}[|\mathcal{Y}|^2] \approx 2\sigma_{ {noise}}^4 + 4 P_{{tx}_j} g_j^2 \sigma_{ {noise}}^2$, respectively. The value $\mathcal{Z}=\epsilon - P_{{tx}_j} g_j^2$ comprehensively quantifies the antagonistic relationship between the concealment capability of UAV communication and the detection capability of Willie. To simplify the analysis, supposing that Willie uses coherent detection, the detection probability can be modeled as
\begin{equation}
p_j = Q\left({\mathcal{Z}}/{\sigma_{ {noise}}} \right),
\end{equation}
where \( Q(\cdot) \) is the standard normal distribution function.

\subsection{Alice Model}
Increasing the transmit power improves the communication throughput, but also increases the risk of being detected by Willie. To  prevent Willie from discovering the interaction between Alice and the UAV \( j \), both parties need to strike a balance between throughput and covertness. The utility of Alice is defined as
\begin{equation}
V_j = \mu \left(\psi \log_2 \left( 1 + {P_{{tx}_j} g_j}/{N_0} \right) - \omega p_j - R_j\right),
\label{eq:alice_utility}
\end{equation}
where \( N_0 \) denotes the channel gain of the environmental noise. \( \mu \), \( \psi \) and \( \omega \) are constants.

\section{SG-based Incentive Mechanism for Covert Communication}
\label{sec:3}
In practical scenarios, the GU and UAVs aim to maximize their respective utilities. However, due to their differing profit orientations and decision priorities, centralized optimization approaches often fail to address the resulting challenges effectively. The SG framework is well suited for this context, as it captures the GU’s dominant role as the system designer and the UAVs’ responsive behavior as service providers \cite{li2025satisfaction}. Within this hierarchical decision-making structure, the GU first establishes incentive strategies to guide UAV behavior, while the UAVs subsequently optimize their transmission power allocation based on the GU’s decisions. This sequential process ensures that the system reaches a stable and predictable equilibrium. Accordingly, the interaction between Alice and the UAVs is modeled as a SG. The SG-based incentive mechanism for the covert communication algorithm is summarized in Algorithm~\ref{alg01}. Alice, as the leader, determines the reward policy \( r_j \), while UAVs, as followers, respond by selecting their \( P_{{tx}_j} \). The SG is defined as $\Omega = \{ ({Alice} \cup \mathcal{J}), (P_{{tx}_j},r_j), (U_j, V_j) \}$, where \(({Alice} \cup \mathcal{J})\) represents the set of all members participating in SG, \((P_{{tx}_j}, r_j)\) denotes the policy set. We denote the optimal power vector provided by UAVs as \(P_{tx}^* = [P_{{tx}_1}^*, \ldots, P_{{tx}_j}^*, \ldots, P_{{tx}_J}^*]\), and the optimal reward as \(r^*\).

\begin{algorithm}[tbp]
	\caption{SG-based incentive mechanism for covert communication in UAV network}
	\label{alg01}
	\begin{algorithmic}[1] 
        \Require $g_i$, $\varphi_i$, $R_{max}$, $\sigma_{{noise}}^2$, $N_0$, $g_j$, $r_0$, $J$, $\zeta$
		\Ensure $P_{tx}^*$, $r^*$, $w^{\ast}$, $V(w^*)$
        \State Initialize the UAV combinations $\mathcal{W}$, $(P_{tx}^*,r^*)$ 
		\For{$w=1,\ldots,\mathcal{W}$}
        \For{each UAV $j$ in ${w}$}
        \State Calculate $U_j\left( P_{{tx}_j} \mid r_0 \right)$ and  $P_{{tx}_j}^*(r_0)$ by Eq.\eqref{eq:utility}
        \State Calculate $V_j\left( P_{{tx}_j}^* \mid r_j \right)$ and $r^*$ by Eq.\eqref{eq:alice_utility}
        \State Update $P_{tx}^* \leftarrow P_{{tx}_j}(r^*)$
        \State \textbf{until} $\lVert r^* - r_0 \rVert < \zeta$
		\EndFor
        \State Calculate combination utility $V(w)$
		\EndFor
        \State Select $w^* \gets \underset{w \in \{1, \ldots, \mathcal{W}\}}{argmax}\; V(w)$
	\end{algorithmic}
\label{alg1}
\end{algorithm}

\begin{definition}
Stackelberg Equilibrium (SE). The policy pair \( (P_{{tx}_j}^*,\, r_j^*) \) constitutes a SE if and only if there exist optimal \( P_{{tx}_j}^* \) and \( r_j^* \) such that:
\begin{equation}
\begin{cases}
\forall P_{{tx}_j}, \quad U_j(P_{{tx}_j}^*, r_j^*) \geq U_j(P_{{tx}_j}, r_j^*) \\
\forall r_j, \quad V_j(P_{{tx}_j}^*, r_j^*) \geq V_j(P_{{tx}_j}^*, r_j)
\end{cases}.
\end{equation}
\end{definition}

We use backward induction to analyze the UAVs’ optimal decisions by computing the first and second derivatives of \( U_j \) with respect to \( P_{{tx}_j} \), as follows:
\begin{equation}
\begin{aligned}
{\partial U_j}/{\partial P_{{tx}_j}} &= {r_j}/({1 + P_{{tx}_j}}) - \phi_j, 
\end{aligned}
\end{equation}
\begin{equation}
\begin{aligned}
{\partial^2 U_j}/{\partial P_{{tx}_j}^2} &= - {r_j}/{(1 + P_{{tx}_j})^2}.
\end{aligned}
\end{equation}

Since the second derivative is always negative, \(U_j\) is strictly concave in \(P_{{tx}_j}\). The optimal solution is $P_{{tx}_j}^* = r_j / \phi_j - 1$. By setting both upper and lower bounds on the transmit power, the optimal power response policy can be expressed as:
\begin{equation}
P_{{tx}_j}^* =
\begin{cases}
P_{{tx}_{j, max}}, & P_{{tx}_j} > P_{{tx}_{j, max}} \\
\frac{r_j}{\phi_j} - 1, & P_{{tx}_{j, min}} < P_{{tx}_j} < P_{{tx}_{j, max}} \\
P_{{tx}_{j, min}}, & P_{{tx}_j} < P_{{tx}_{j, min}}
\end{cases}.
\end{equation}

\begin{figure*}[htb]
    \centering
\includegraphics[width=\textwidth]{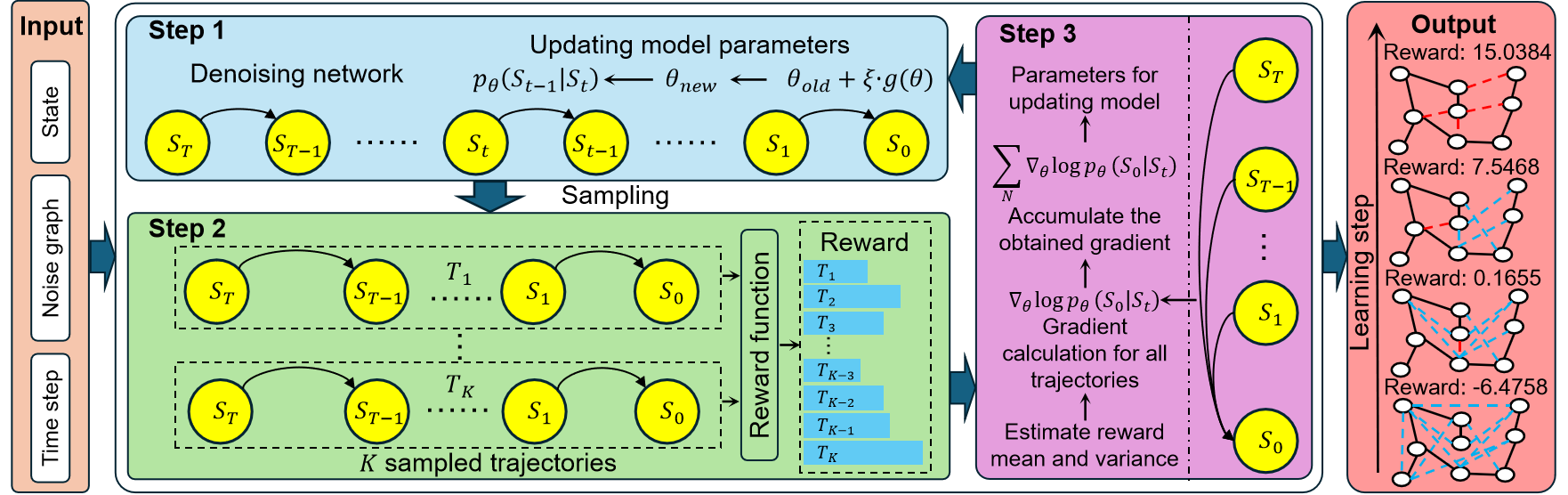}%
	\caption{UAV network topology optimization method based on GDPO. In output topologies, blue dashed lines represent redundant links, while red dashed lines indicate beneficial links. As training progresses, the reward steadily increases, demonstrating continuous improvement of the graph structure. }
	\label{Fig02}
\end{figure*}

Substituting \(P_{{tx}_j}^*\) into \eqref{eq:alice_utility}, we can acquire
\begin{equation}
\begin{aligned}
\frac{\partial V_j}{\partial r_j} 
&= \mu \left\{
    \frac{\psi g_j}{\phi_j N_0 \ln 2} \cdot 
    \frac{1}{1 + \frac{g_j}{N_0}  P_{{tx}_j}^* } \right. \\
&\quad \left. + \frac{\omega g_j^2}{\phi_j \sigma_{ {noise}}} \cdot 
    Q'\left( \frac{ \mathcal{Z}^*  }{ \sigma_{ {noise}} } \right)
\right\} 
- \ln\left( \frac{r_j}{\phi_j} \right) - 1, 
\end{aligned}
\end{equation}

\begin{equation}
\begin{aligned}
\frac{\partial^2 V_j}{\partial r_j^2} 
&= \mu \Bigg\{ 
 -\frac{ \psi g_j^2 }{ \phi_j^2 N_0^2 \ln 2 } \cdot 
 \frac{1}{ \left[ 1 + \frac{g_j}{N_0} P_{{tx}_j}^*  \right]^2 } \\
&\quad + \frac{ \omega g_j^4 }{ \phi_j^2 \sigma_{ {noise}}^2 } \cdot 
\left( \frac{ \mathcal{Z}^*  }{ \sigma_{ {noise}} } \right)
e^{ -\frac{ \left[ \mathcal{Z}^*  \right]^2 }{ 2 \sigma_{ {noise}}^2 } } 
\Bigg\}
- \frac{1}{r_j},
\end{aligned}
\end{equation}
where \(\mathcal{Z}^* \leq 0\), that is \(r \geq \phi_j \left(\frac{\epsilon}{g_j^2} + 1\right)\), the second derivative is strictly negative and \(V_j\) is concave. Thus, the optimal solution can be found via convex optimization.

\begin{theorem}
    There exists a unique SE \((P_{{tx}_j}^*, r_j^*)\) in the SG.
\end{theorem}

\begin{proof}
    Given a reward \(r_j\), each UAV has a unique \(P_{{tx}_j}^*\) due to the concavity of Eqs. \eqref{eq:utility} and \eqref{eq:alice_utility}. Moreover, Alice has a unique optimal policy given the best responses of the UAVs. Therefore, the \((P_{{tx}_j}^*, r_j^*)\) maximizes the utilities of the UAVs and Alice, respectively, and constitutes a unique SE. 
\end{proof}

\section{GDPO-based Topology Optimization}
\label{sec:5}

\subsection{Markov Decision Process}
The UAV network topology optimization problem is formulated as a Markov Decision Process (MDP), which provides a principled framework for modeling sequential decision-making problems. The MDP is defined by a quintuple $(\mathcal{S}, \mathcal{A}, p, r_{ {topo}}, \rho_0)$, which $p$ is the transition function determining the probabilities of state transitions, $\rho _{0}$ gives the distribution of the initial state.

\textbf{State space} : In each time step $t$, the UAV and GU state space can be defined as $\mathcal{S}_t = (\mathcal{J}_t, G_t, E_t)$, where $\mathcal{J}_t=(x_{j_t}, y_{j_t}, h_{j_t}, P_{{tx_j}_t})$, $G_t=(x_{i_t}, y_{i_t})$. $E_t$ represents the set of the UAV-GU links at time step $t$.



\textbf{Action space}: At time step $t$, the action $\mathcal{A}_t$ denotes a transformation $\mathcal{S}_t \rightarrow \mathcal{S}_{t-1}$, reflecting the denoising or graph reconstruction decision. Action defined as a set of link operations $\mathcal{A}_t = \{ (e, a_t \mid e = (u, v) \in {E}_t \}$, where \( u, v \in \mathcal{J} \). The operation $a_t$ means addition, deletion, and maintenance of communication links $e$.




\textbf{Reward function} : We define a multi-objective joint optimization problem that aims to improve network coverage, energy consumption, and network connectivity. The reward function is defined as
\begin{equation}
r_{ {topo}} = \alpha \cdot r_{ {cov}} - \beta \cdot r_{ {ener}} - \gamma \cdot r_{ {conn}} - \delta \cdot r_{ {over}},
\end{equation}
where $r_{ {cov}}$, $r_{ {ener}}$, $r_{ {conn}}$, and $r_{ {over}}$ denote the rewards for UAV coverage, energy consumption, network connectivity, and GU overlapping coverage, respectively. The coefficients $\alpha$, $\beta$, $\gamma$, and $\delta$ represent the corresponding weights. The coverage reward is $r_{ {cov}} = \frac{1}{I} \sum_{i=1}^{I} D_{ {cov}}(i)$, where $D_{ {cov}}(i) = 1$ when $P_{r_i} \geq P_{min}$ and 0 otherwise. Here, $D_{ {cov}}(i)$ indicates whether GU $i$ is covered, $P_{r_i}$ is the signal power received by GU $i$, and $P_{min}$ is the minimum required power threshold. 
The energy consumption reward is $r_{ {ener}} = a \cdot r_{ {fly}} + b \cdot r_{ {tra}} + o \cdot r_{ {cha}}$, where 
$r_{ {fly}} = (\vartheta + \varrho h_j) t_f$,
$r_{ {tra}} = P_{{tx}_j} t_f$,
$r_{ {cha}} = 2 P_l t_f$, $a$, $b$, $o$, and $\vartheta$ are constants. $r_{ {fly}}$, $r_{ {tra}}$, and $r_{ {cha}}$ represent the consumption of UAV flight, signal transmission, and maintenance of communication links, respectively. $P_l$ is the link maintenance power, $\varrho$ is the altitude power coefficient, and $t_f$ is the flight time.
The connectivity reward is $r_{ {conn}} = 
\begin{cases}
0, & \|C\| = 1 \\
100, & \|C\| \neq 1
\end{cases}$, where $\|C\|$ represents the number of components connected in the UAV topology. $\|C\| = 1$ indicates full connectivity, while $\|C\| \neq 1$ indicates the presence of isolated subnetworks.
The coverage reward is $r_{ {over}} = \sum_{i=1}^{I} r_{ {over}}^{(i)}$, where $r_{ {over}}^{(i)} = 
\begin{cases}
5(m_i - 1), & m_i \geq 2 \\
0, & m_i < 2
\end{cases}$. Here, $r_{ {over}}^{(i)}$ denotes the reward for multiple UAVs covering the same GU $i$, and $m_i$ is the number of UAVs covering GU $i$.

\subsection{Topology Optimization Process}
In the topology graph representation, UAVs and communication links are treated as nodes and edges, respectively. Performance-related metrics of the network topology are used as a reward function to guide the GDPO model for policy learning. This graph-based modeling approach enables GDPO to demonstrate a stronger generalization ability and sample efficiency in complex communication scenarios with randomly distributed nodes and dynamic link conditions, showing great potential for practical applications. As shown in Fig.~\ref{Fig02}, the input to GDPO-based topology optimization includes the time step, the noisy graph, and the graph state used to initialize the denoising process. The output consists of multiple generated topologies and their corresponding reward values. The detailed algorithm is provided in \textbf{Algorithm \ref{alg02}} and proceeds as follows:

\textbf{Step 1 (Lines 2-6).} In each training round, multiple graph generation trajectories are sampled from the diffusion model. This involves gradually restoring clear graph structures from noisy states using the denoising network. The reward for each final graph is then computed, enabling initial updates of the GD model parameters.

\textbf{Step 2 (Lines 7-9).} Multiple time steps are randomly selected to sample multiple state trajectories from the current model. The final graph structure of each trajectory is evaluated using the reward function. All rewards are normalized, and the policy gradient is calculated based on the log-probability gradient.

\textbf{Step 3 (Line 10).} Through multiple iterations, the quality of the generated graph structures under the reward function is continuously improved. Gradients for all trajectories are computed and accumulated, and the model parameters are further optimized according to the learning rate.

To efficiently optimize the parameter, an efficient estimation method called Eager Policy Gradient (EPG) enables the reward signal of the final structure to directly influence the gradient update at each timestep in the GD model, thus improving training stability and convergence speed. Compared to the REINFORCE policy gradient method \cite{sutton1998reinforcement}, this mechanism exhibits better performance in handling high-dimensional sparse spaces and high-variance estimation problems often encountered in graph structure generation. In the topology generation, the EPG update in GDPO is defined as
\begin{equation}
g(\theta) \triangleq \frac{1}{K} \sum_{k=1}^{K} \frac{T}{|T_k|} \sum_{t \in T_k} r_{ {topo}}(S_0^k) \nabla_{\theta} \log p_{\theta}(S_0^k \mid S_t^k),
\label{eq:estimate_gradient}
\end{equation}
where $\theta$ represents the model parameters, $K$ denotes the total number of sampled trajectories, $T$ is the number of denoising time steps, and $\{T_k \subset (1, T)\}_{k=1}^{K}$ represents a random subset of time steps. $S_0^k$ represents the initial topology graph of the $k$-th trajectory, $S_t^k$ denotes the topology graph at time step $t$, and $r_{ {topo}}(S_0^k)$ is the normalized reward of the initial topology graph, defined as $r_{ {topo}}(S_0^k) = \frac{r_k - \bar{r}}{ {std}[r]}$, where $\bar{r} = \frac{1}{K} \sum_{k=1}^{K} r_k$ and $\textit{std}[r] = \sqrt{\frac{\sum_{k=1}^{K} (r_k - \bar{r})^2}{K - 1}}$.


\begin{algorithm}[tbp]
	\caption{GDPO for UAV Topology Optimization}
	\label{alg02}
	\begin{algorithmic}[1] 
		\Require $p_{\theta}$, $\mathit{T}$, $\left|\mathcal{T}\right|$, $r_{ {topo}}(\cdot)$, $\mathit{K}$, $\eta$, $\mathit{N}$ 
		\Ensure Optimized topology graph
		\For{$i=1,\ldots,N$}
		\For{$k=1,\ldots,K$}
        \State $\boldsymbol{S}_{0:T}^{k}\sim p_{\theta}$	\Comment{Sample trajectory}
		\State $\mathcal{T}_{k}\sim\mathrm{Uniform}(1,T)$	\Comment{Sample timesteps}
		\State $r_k\leftarrow r_{ {topo}}(\boldsymbol{S}_0^{k})$	\Comment{Get rewards}
		\EndFor
        \State Update state of UAV using \textbf{Algorithm \ref{alg01}}
		\State Estimate the EPG by Eq. \eqref{eq:estimate_gradient}
		\State Update GDPO model parameter by Eq. \eqref{eq:update_modelparameters}
		\EndFor
	\end{algorithmic}
\label{alg1}
\end{algorithm}

The model parameters are updated as
\begin{equation}
\theta \leftarrow \theta + \eta \cdot g(\theta),
\label{eq:update_modelparameters}
\end{equation}
where $\eta$ is the learning rate.


\begin{figure*}[!t]
	\centering
	\begin{minipage}{0.32\linewidth}
		\centering
		\includegraphics[width=0.8\linewidth]{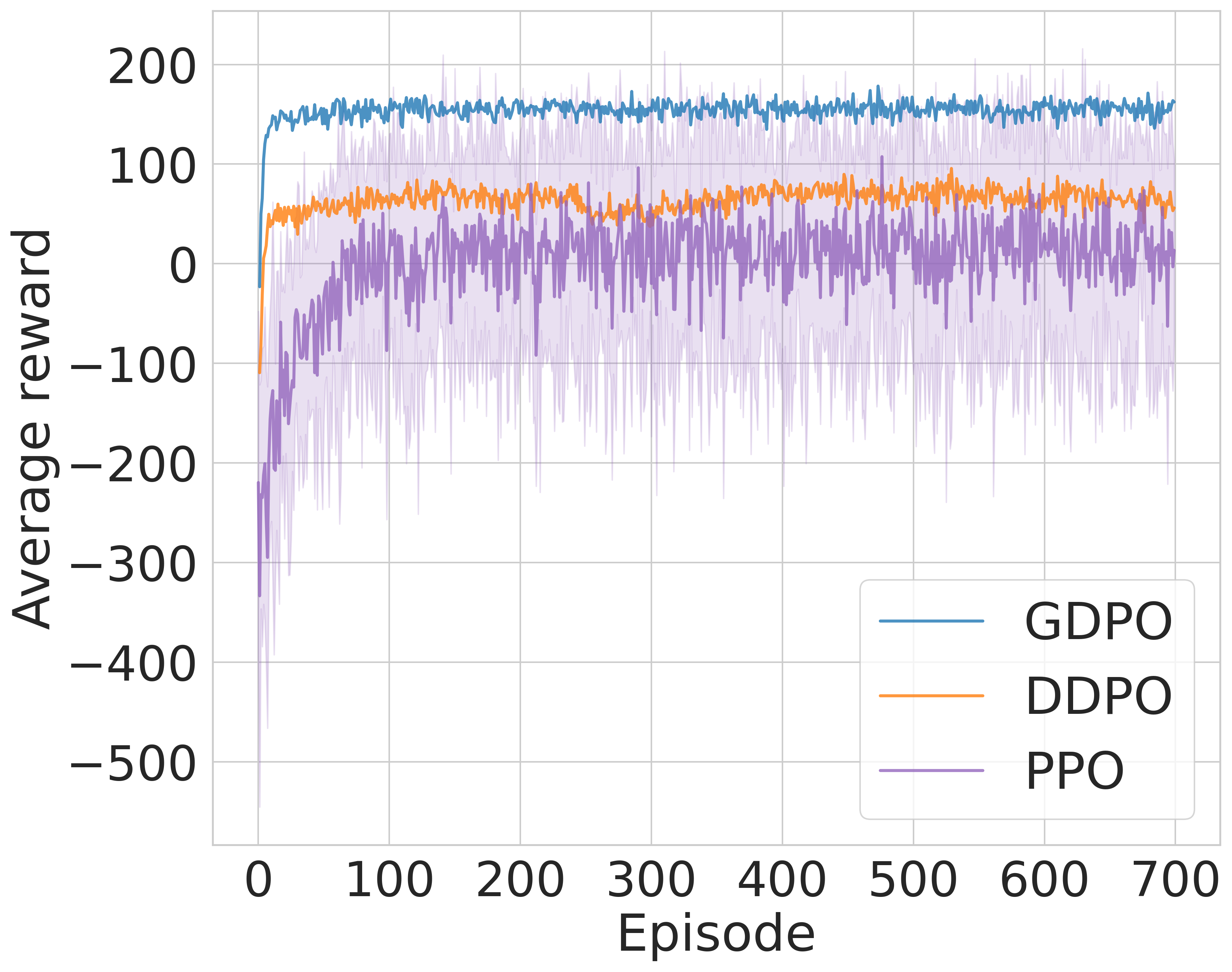}
		\caption{Reward vs. different methods.}
		\label{Fig03}
	\end{minipage}%
	\hfill
	\begin{minipage}{0.32\linewidth}
		\centering
		\includegraphics[width=0.8\textwidth]{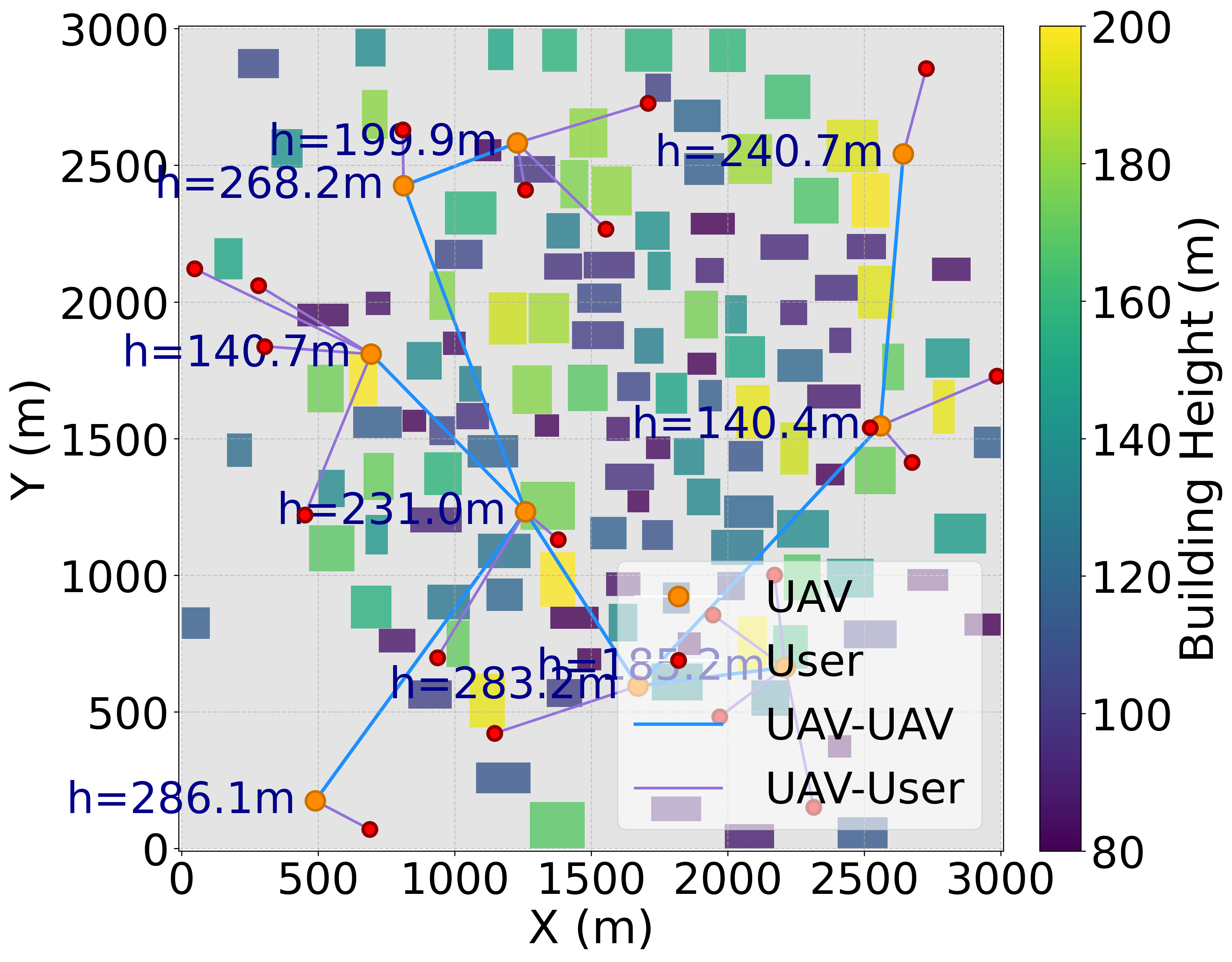}
		\caption{Topology generation }
		\label{Fig04}
	\end{minipage}%
	\hfill
	\begin{minipage}{0.32\linewidth}
		\centering
		\includegraphics[width=0.76\linewidth]{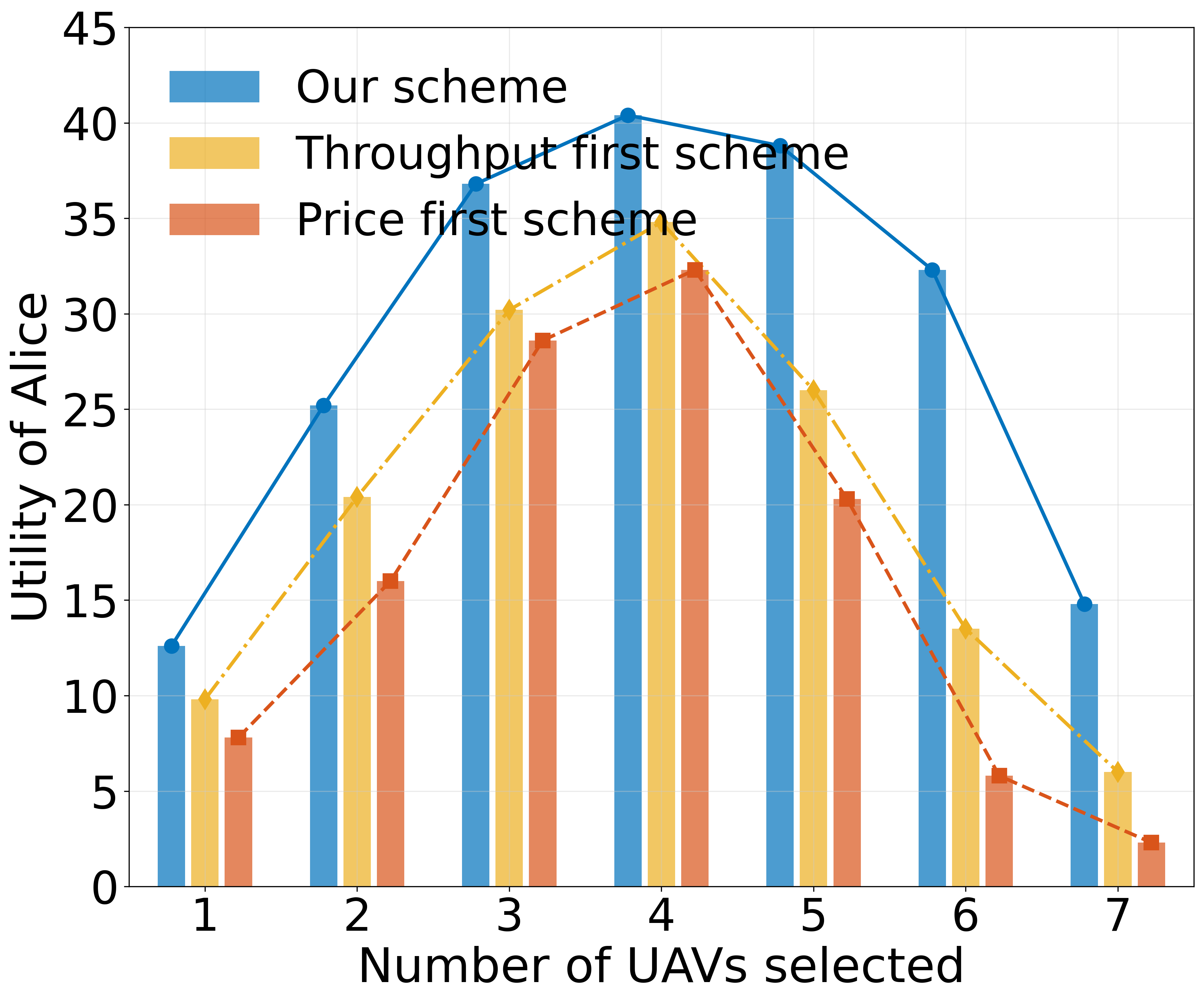}
		\caption{Utility vs. incentive mechanisms.}
		\label{Fig05}
	\end{minipage}
    \vspace{-0.4cm}
\end{figure*}

\section{Experiment evaluation}
\label{sec:6}
We consider GUs randomly distributed in a 3km × 3km urban low-altitude scenario, where the spatial positions of UAVs are determined by system parameters and a limited spatial area. To illustrate the advantages of employing GDPO for UAV network topology optimization in this study, Proximal Policy Optimization (PPO) and Dynamic Diffusion Policy Optimization (DDPO) are selected as benchmark algorithms for comparison. The parameters and their corresponding values are presented in Table \ref{tab:table1}.
\begin{table}[t]
	\caption{Experiment Parameters\label{tab:table1}}
	\centering
	\begin{tabular}{p{1.2cm}|p{2.4cm}|p{1.2cm}|p{2.4cm}}
		\hline
		Symbol & Value (Unit) & Symbol & Value (Unit) \\
		\hline
		$J$ & $9$ & $N_0$ & $1\,\text{dB}/\text{Hz}$ \\
		$I$ & $20$ & $g_i$ & $[3\,\text{dB}, 10\,\text{dB}]$ \\
		$f_c $ & $2.4\,\text{GHz}$ & $g_j$ & $1\,\text{dB}$ \\
		$P_{{tx}_{min}}$ & $10\,\text{dBm}$ & $\epsilon$ & $0.6$ \\
		$P_{{tx}_{max}}$ & $30\,\text{dBm}$  & $\sigma_{ {noise}}$ & $0.1\,\text{W}$ \\
		$P_{min}$ & $-90\,\text{dBm}$ & $h_j$ & $100 \sim 300\,\text{m}$ \\
		\hline
	\end{tabular}
\end{table}


Fig.~\ref{Fig03} shows the average rewards of GDPO, DDPO, and PPO during training. PPO converges the slowest and has the largest fluctuations, indicating high instability. DDPO converges faster than PPO but still suffers from considerable oscillations and an unsmooth learning process. In contrast, GDPO matches DDPO's convergence speed but with significantly lower post-convergence variance, demonstrating stronger stability. This improvement stems from GDPO's integration of GD sampling with policy optimization, which enhances exploration, reduces gradient variance, and better balances multiple objectives. As a result, GDPO achieves superior policy learning and convergence in dynamic environments, making it well suited for multi-objective optimization.

Fig.~\ref{Fig04} shows the link structure of a UAV network topology generated after 50 iterations in a space-constrained environment. The topology includes UAV–UAV and UAV–User links. In the initial formation stage, excessive link redundancy among UAVs increases energy consumption and undermines network stability. In contrast, GDPO adds high-efficiency links and removes redundant ones, producing an adaptive topology that meets dynamic user coverage and robust connectivity requirements in constrained spaces.

Fig.~\ref{Fig05} compares Alice’s utility under different incentive schemes with a maximum budget constraint $(Y_{max} = 50)$. When selecting four UAVs, all schemes reach peak utility, indicating an optimal balance between UAV count and budget. Although adding UAVs initially improves utility, it declines beyond the peak due to the budget limit. Notably, our approach consistently outperforms the others, as throughput-priority and cost-priority schemes fail to balance the utility function, leading to competition and lower overall gains compared to our integrated method.

\section{Conclusion}
\label{sec:7}
This paper proposes a self-organizing UAV network framework combining GDPO with a SG-based incentive mechanism. The GDPO method uses generative AI to dynamically generate sparse but well-connected topologies, enabling flexible adaptation to changing node distributions and task demands. Meanwhile, the SG guides self-interested UAVs to select optimal relay behaviors and neighbor links, sustaining cooperation and enhancing covert communication. This integrated approach improves network resilience, scalability, and covertness under dynamic, distributed conditions. Future work will focus on incorporating radio spectrum distribution awareness into UAV network optimization to further strengthen the network’s communication reliability and assurance.



\section*{Acknowledgment}

This work was supported in part by the Guangxi Natural Science Foundation of China under Grant 2025GXNSFAA069687, and in part by the Science and Technology Key Project of Guangxi Province AB23026038, and in part by the Graduate Study Abroad Program of GUET GDYX2024001.

\normalem
\vspace{-0.2cm}
\bibliographystyle{IEEEtran}
\bibliography{main}

\end{document}